\documentclass[10pt,twocolumn,letterpaper]{article}

\usepackage{iccv}

\usepackage[dvipsnames]{xcolor}
\makeatletter
\@namedef{ver@everyshi.sty}{}
\makeatother
\usepackage{tikz,pgfplotstable,filecontents,graphicx}
\usetikzlibrary{calc, math, arrows.meta, external, fadings}
\pgfplotsset{compat=1.16}

\usepackage{times}
\usepackage{epsfig}
\usepackage{graphicx}
\usepackage{amsmath}
\usepackage{amssymb}
\usepackage{amsthm}
\usepackage{multirow}
\usepackage{enumitem}
\usepackage{xifthen}
\usepackage{mathtools}
\usepackage{arydshln}
\usepackage{booktabs}
\usepackage{comment}

\theoremstyle{plain}
\newtheorem{thm}{Theorem}
\newtheorem{lemma}[thm]{Lemma}
\theoremstyle{definition}
\newtheorem{defn}{Definition}

\usepackage{subdepth} 

\makeatletter
\newcommand{\myparagraph}{%
  \@startsection{paragraph}{4}%
  {\z@}{1ex \@plus 1ex \@minus .2ex}{-1em}%
  {\normalfont\normalsize\bfseries}%
}
\makeatother

\usepackage[pagebackref=true,breaklinks=true,letterpaper=true,colorlinks,bookmarks=false]{hyperref}

\usepackage{xcolor}

\newcommand{\ik}{\kappa}
\newcommand{\im}{\mu}
\newcommand{\iu}{u}
\newcommand{\sw}{\text{sw}}
\newcommand{\marg}{\text{m}}

\DeclareMathOperator*{\rk}{rank}
\DeclareMathOperator*{\argmin}{arg\,min}
\DeclareMathOperator{\SE}{SE}

\newcommand{\sqrtvo}[1][]{\ifthenelse{\equal{#1}{}}{$\sqrt{\textit{VO}}$}{$\sqrt{\textit{VO}}$-{#1}}}
\newcommand{\sqrtvio}[1][]{\ifthenelse{\equal{#1}{}}{$\sqrt{\textit{VIO}}$}{$\sqrt{\textit{VIO}}$-{#1}}}
\newcommand{\sqvo}[1][]{\ifthenelse{\equal{#1}{}}{$\textit{VO}$}{$\textit{VO}$-{#1}}}
\newcommand{\sqvio}[1][]{\ifthenelse{\equal{#1}{}}{$\textit{VIO}$}{$\textit{VIO}$-{#1}}}

\iccvfinalcopy 

\def\iccvPaperID{7763} 

\begin{document}


\title{Square Root Marginalization for Sliding-Window Bundle Adjustment}

\author{Nikolaus Demmel \quad David Schubert \quad Christiane Sommer \quad Daniel Cremers \quad Vladyslav Usenko\\
Technical University of Munich\\
{\tt\small \{nikolaus.demmel,d.schubert,c.sommer,cremers,vlad.usenko\}@tum.de}}

\maketitle
\ificcvfinal\thispagestyle{empty}\fi

\ificcvfinal
\let\thefootnote\relax\footnote{This work was supported by the ERC Advanced Grant SIMULACRON and by the German Science Foundation Grant CR 250/20-1 ``Splitting Methods for 3D Reconstruction and SLAM''.}
\fi

\begin{abstract}
  In this paper we propose a novel square root sliding-window bundle adjustment suitable for real-time odometry applications.
  The square root formulation pervades three major aspects of our optimization-based sliding-window estimator: for bundle adjustment we eliminate landmark variables with nullspace projection;
  to store the marginalization prior we employ a matrix square root of the Hessian; and when marginalizing old poses we avoid forming normal equations and
  update the square root prior directly with a specialized QR decomposition.
  We show that the proposed square root marginalization is algebraically equivalent to the conventional use of Schur complement (SC) on the Hessian.
  Moreover, it elegantly deals with rank-deficient Jacobians producing a prior equivalent to SC with Moore–Penrose inverse.
  Our evaluation of visual and visual-inertial odometry on real-world datasets demonstrates that the proposed estimator is 36\% faster than the baseline.
  It furthermore shows that in single precision, conventional Hessian-based marginalization leads to numeric failures and reduced accuracy.
  We analyse numeric properties of the
  marginalization prior to explain why our square root form does not suffer from the same effect and therefore entails superior performance. 
\end{abstract}

\section{Introduction}

\begin{figure}[t]
\begin{center}
   \includegraphics[trim=0 0.2em 0 3.3em,clip,width=0.99\linewidth]{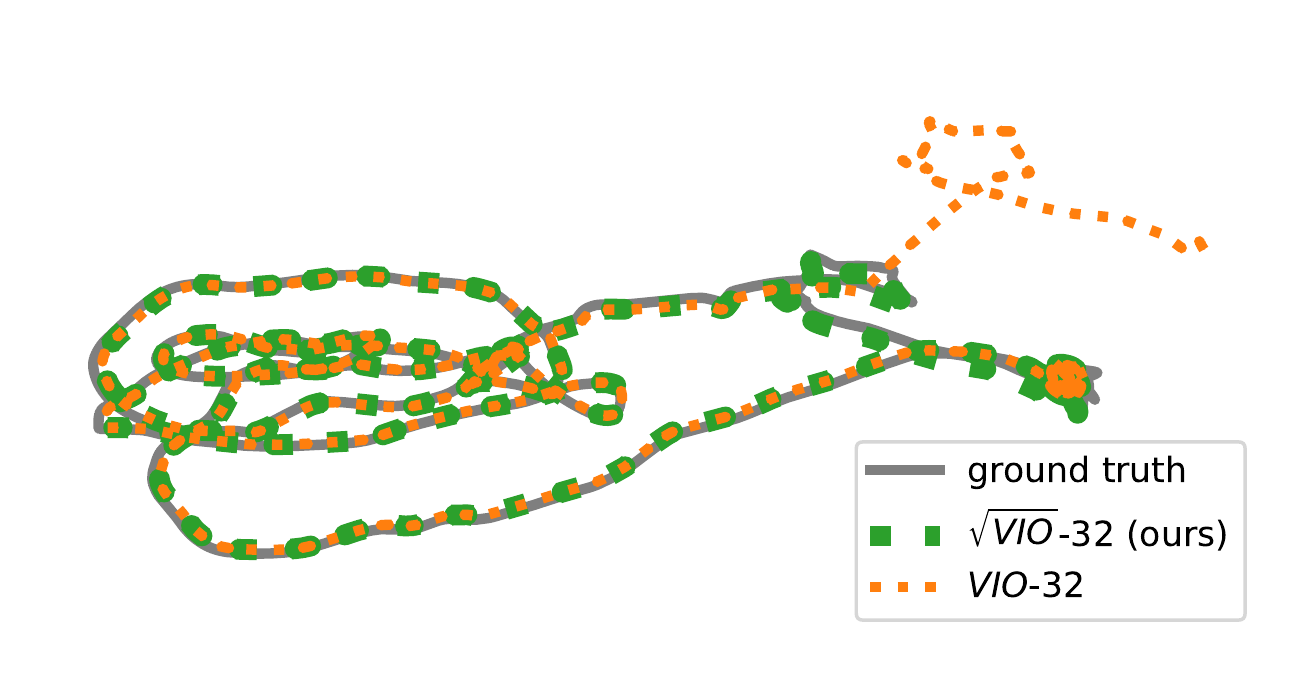}\\
   \vspace{-0.9em}
   \includegraphics[trim=0 1.3em 0 0.9em,clip,width=0.99\linewidth]{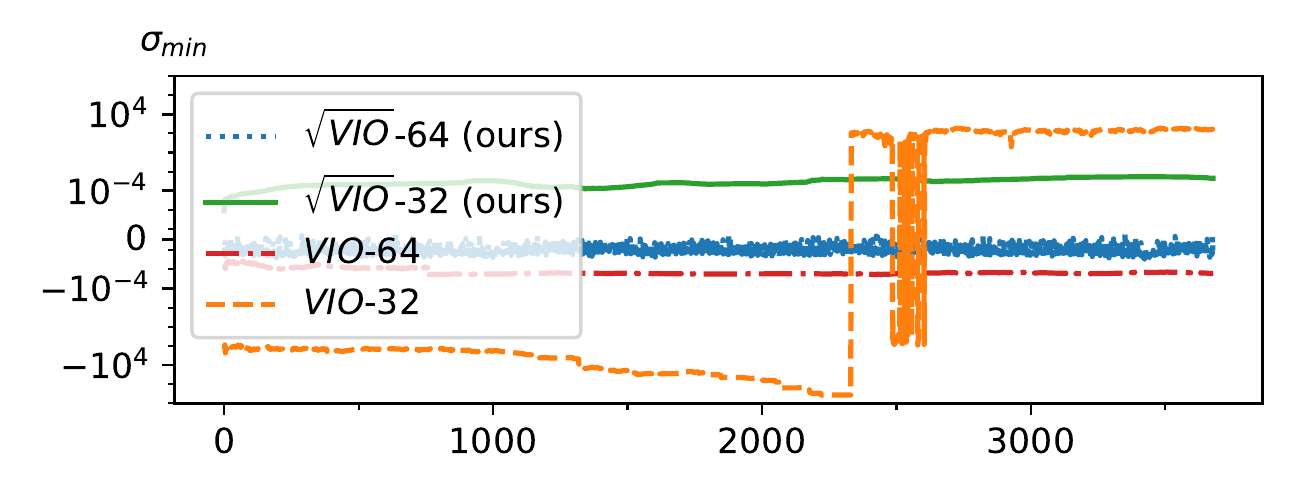}
\end{center}
\caption{\emph{Top:} estimated visual-inertial odometry trajectories on the \emph{eurocMH01} sequence. 
The conventional baseline \sqvio[64] works well with double-precision floats, but fails in single precision (\sqvio[32]). 
In contrast, the proposed square root estimator \sqrtvio[32] even in single precision retains the same accuracy. It processes the whole sequence in $26\,\mathrm{s}$ on a desktop CPU (7.1x faster than real time).
\emph{Bottom:} smallest eigenvalue $\sigma_{\min}$ of the marginalization prior Hessian $H_\marg$ evolving over time (linear y-axis for  $|\sigma_{\min}| < 10^{-8}$, logarithmic elsewhere).
We expect values close to zero (positive semi-definite Hessian with gauge freedom). 
While the conventional (squared) formulation in single precision leads to negative eigenvalues with large magnitude, accumulating error, and (ultimately) numerical failure, the proposed square root approach has $\sigma_{\min}$ of bounded magnitude and remains stable.
}
\label{fig:teaser}
\end{figure}

Visual odometry has been a key component in environment mapping, robot navigation and autonomous systems for a long time.
With low-cost devices, such as smartphones or robot vacuum cleaners becoming increasingly prevalent in our daily lives, we see a growing  need to solve odometry problems in a fast and robust manner. In addition, scalable solutions on specialized hardware require algorithms to run with limited floating point precision.

To keep the system size bounded with a fixed number of state variables over time, marginalization is a commonly used technique,
where the remaining sub-problem can be interpreted in terms of a marginal distribution with the same solution as before.
However, for the implementation of the marginalization prior as well as the solution of the associated optimization problem, there are multiple options.
The Schur complement technique is an easy-to-implement choice that many state-of-the-art odometry and SLAM systems employ, but it relies on the Hessian matrix of the linearized system.
While this is not a problem in many applications, the fact that the square of the Jacobian and thus a squared condition number are involved in SC may lead to numerical instabilities.
In Kalman filter literature, this has often been addressed by using a square root filter approach \cite{wu2015square}. Moreover, it has recently been shown that even in an optimization-based bundle adjustment setting, matrix square roots can be exploited to increase numerical stability \cite{demmel2021square}.
Following up on these findings, we propose  an optimization-based formulation of visual (inertial) odometry that uses matrix square roots in the optimization stage as well as to store and update the marginalization prior.
We dub the two flavors of our method \sqrtvo{} and \sqrtvio{} for purely visual and visual-inertial odometry, respectively.
The contributions of our paper are as follows:

\begin{itemize}[noitemsep, topsep=-0.3em]

    \item We propose a novel square root formulation for optimization-based sliding-window estimators. 
    
    \item One of the components of this square root formulation is a specialized QR decomposition and we prove the analytical equivalence with marginalization based on SC and Cholesky decomposition.

    \item Our QR marginalization naturally includes the case of rank-deficient Jacobians.

    \item
    The proposed square root formulation enables optimization-based VO and VIO with single-precision floating point computations without loss of accuracy.
    
    \item On several real-world datasets, we systematically analyze the effect of using nullspace projection combined with our specialized QR decomposition on runtime and accuracy.  The proposed square root estimator is $23\%$ faster than the baseline and $36\%$ faster when additionally switching to single precision.
    
    \item We release our implementation as open source:\\
    \url{https://go.vision.in.tum.de/rootvo}.
\end{itemize}

\section{Related work}

In what follows, we review relevant literature in visual and visual-inertial odometry, discuss filter-based and optimization-based techniques, and highlight works dealing with rank-deficient Jacobians.

\myparagraph{Filter-based state estimation for VIO}
More than 60 years ago, Kalman proposed a filter~\cite{kalman1960new} that set the basis for state estimation in many different applications, visual-inertial odometry being one of them.
Many variations of the Kalman filter have been developed, introducing different improvements to the original, where one prominent example for VIO is the MSCKF \cite{mourikis2007}.
Square root filters are able to improve numerical stability and let the system run on single-precision hardware~\cite{maybeck1982stochastic, bierman2006factorization, dellaert2006square, wu2015square} and in particular information filters are closely related to optimization-based approaches \cite{strasdat2012why}. 
As such, the proposed approach shares many of these advantages with the square root inverse variant of the MSCKF \cite{wu2015square}, but ours can relinearize residuals until they leave the sliding window and accommodate rank-deficient marginalization priors. 
Yang et al.\ were the first to prove the equivalence of Schur complement and nullspace marginalization (a technique using matrix square roots) from a Kalman filter perspective~\cite{yang2017null}.

\myparagraph{Optimization-based visual (inertial) odometry systems}
Optimization-based approaches have become the state-of-the-art in visual and visual-inertal odometry. Usually, they are implemented as a fixed-lag smoother, where old or redundant state variables are removed from the optimization window in order to keep the system real-time capable~\cite{sibley2010planetary,leutenegger2015keyframe,engel2017direct,qin18vins,usenko2019visual}. 
All mentioned approaches rely on marginalization using the Schur complement, which means the marginalization energy is stored as a quadratic function of the remaining variables. Recent work, however, has shown that it can be beneficial to perform marginalization in square root form~\cite{demmel2021square}, both numerically and in terms of runtime. While \cite{demmel2021square} only shows this for batch optimization with temporary landmark marginalization, we introduce an efficient way to permanently marginalize frame variables.

\myparagraph{Rank-deficient Jacobians}
Most works on bundle adjustment and odometry systems tacitly assume full-rank Jacobians and information matrices, but this may not always be the case.
Generalized inverses, generalized Schur complements and rank-deficient information matrices have been analyzed for years outside the BA or SLAM context~\cite{rohde1965generalized, marsaglia1974rank, burns1974generalized, carlson1974generalization, ouellette1981schur, pukelsheim1990information, liu2010generalized}.
Only few works propose ways to deal with rank-deficiency in SLAM: Mazuran et al.~\cite{mazuran2014nonlinear, mazuran2016nonlinear} introduce an orthogonal projection of all involved matrices that removes the Jacobian degeneracy and
Leutenegger et al.~\cite{leutenegger2015keyframe} use the pseudo-inverse for landmark marginalization in case the Hessian is singular.

\section{Sliding-window bundle adjustment}

\subsection{Problem statement}

\subsubsection{Batch optimization problem}

The goal of odometry is to estimate the system state from a set of visual, inertial, or other measurements. The state $x$ can consist of poses, velocities, biases at different moments of time (all referred to as \emph{frame variables} hereafter), and landmark positions. In this paper we consider visual odometry (VO) and visual-inertial odometry (VIO), but the theory described here is general and can be applied to other sensors (GPS, LIDAR).
If visual measurements are involved, landmark positions are part of the optimization problem, and eventually a bundle adjustment-type problem is solved.

To estimate the optimal state $x$ we can solve the non-linear least squares optimization problem
\begin{align}
\label{eq:e_tot}
    E(x) = \tfrac{1}{2}\|W^{1/2} ~ \hat{r}(x)\|^2 = \tfrac{1}{2}\|r(x)\|^2 \,
\end{align}
where $\hat{r}(x)$ is a vector function of stacked residuals and $W$ is the weighting matrix related to the measurement uncertainty. To simplify notation we absorb $W$ into the weighted residual vector $r(x)$. 
Our implementation is based on \emph{Basalt}~\cite{usenko2019visual} and we follow their residual formulation for the vision and IMU terms.
\eqref{eq:e_tot} is an incrementally evolving batch optimization problem that grows in size and includes measurements and variables up to the current time.

\subsubsection{Sliding-window energy}

To keep the system real-time capable, the optimization is constrained to a small window of recent states and residuals. When new states are added, old ones are removed through marginalization, and we
keep the information from residuals that depend on removed variables in an additional energy term $E_\marg(x)$. Together with the squared vector of active residuals $r_\text{a}$, this forms the sliding-window energy
\begin{equation}
\label{eq:e_sw}
    E_\sw(x) = \tfrac{1}{2}\Vert r_\text{a}(x) \Vert^2 + E_\marg(x)\,.
\end{equation}
From here on, $x$ contains only those variables which are active in the current window.
By repeating this process of marginalizing variables when new measurements and states are added, we can keep the problem size fixed.
We refer to this type of marginalization, which allows for keeping the problem size small, as \emph{permanent} marginalization, since the marginalized variables are not reintroduced to the problem.
Using permanent marginalization, we describe how \eqref{eq:e_sw} transforms from one point in time to the next in Sec.~\ref{sec:marginalization}.

\subsubsection{Hessian versus square root form}
Most commonly, the marginalization energy is stored as a quadratic form of the active optimization variables:
\begin{align}
\label{eq:emarghess}
    E_\marg(x) = \tfrac{1}{2}(x - x^0)^\top H_\marg (x - x^0) + b_\marg^{\top}(x - x^0)\,,
\end{align}
where $x^0$ is the linearization point. Note, that strictly speaking $x$ here only denotes a subset of the variables 
and $H_\marg$ and $b_\marg$ are of according size (see Sec.~\ref{sec:marginalization} for details).

It is also possible to store this energy in square root form:
\begin{align}
\label{eq:emargsqrt}
    E_\marg(x) = \tfrac{1}{2}\|r_\marg + J_\marg (x - x^0)\|^2 \,.
\end{align}
Up to an additive constant, this is the same as \eqref{eq:emarghess} and we can relate the two representations by
\begin{align}
H_\marg = J_\marg^\top J_\marg\,,\qquad b_\marg = J_\marg^\top r_\marg\,.
\end{align}
Note, that for a given $H_\marg$ and $b_\marg$, $J_\marg$ and $r_\marg$ are not unique.

Shifting the expansion point by $\delta x$ away from $x^0$ means updating $r_\marg$ by $r'_\marg = r_\marg + J_\marg \delta x$, while $J'_\marg=J_\marg$. The equivalent shift in \eqref{eq:emarghess} can be done using
\begin{align}
\label{eq:shift}
    b'_\marg = b_\marg + H_\marg\delta x\,,\quad H'_\marg = H_\marg\,.
\end{align}

\subsubsection{Optimization}
\label{sec:optimization}

To obtain a state estimate, we minimize the energy in \eqref{eq:e_sw},
\begin{align}
    \hat{x} = \argmin_x E_\sw(x)\,,
\end{align}
using Levenberg-Marquardt optimization. To do this efficiently, we exploit the sparsity of the landmark Jacobian:
as is usual in bundle adjustment, we use \emph{temporary} marginalization of all landmarks to significantly reduce the problem size, then solve the reduced camera system, and finally perform back substitution to recover the landmark positions.

\subsection{Marginalization}
\label{sec:marginalization}

Once the energy in \eqref{eq:e_sw} has been optimized for the current window, and before a new frame is added, variables are chosen for marginalization.
We now derive how $E_\marg$ is calculated for the upcoming window of the next time step with the traditional Hessian-based approach, followed by the proposed equivalent square root formulation in Sec.~\ref{sec:sqrtmarg}.

\subsubsection{Linearization}

Using linearized residuals, we can approximate the energy in \eqref{eq:e_sw} in terms of a perturbation $\Delta x$ from the current (optimal) state estimate $x$ as
\begin{align}
    E_\text{lin}(\Delta x) = \tfrac{1}{2}\Vert r + J\Delta x \Vert^2\,,\\
    r = \begin{pmatrix}r_\text{a}(x)\\r_\marg + J_\marg(x-x^0)\end{pmatrix}\,,
\end{align}
where $x^0$ is the linearization point of the old marginalization prior.
The residual vector $r$ contains the active residuals $r_\text{a}$ and old marginalization residuals shifted to the current state estimate $x$. The Jacobian $J$ contains the Jacobian of $r_\text{a}$ and the marginalization Jacobian $J_\marg$.
This constitutes a system of linear equations, with corresponding normal equations
\begin{align}
\label{eq:normal_eq}
    H \Delta x &= - b\,,
\end{align}
where $H=J^\top J$ and $b=J^\top r$.
We now define a set of variables $x_\im$ that we want to marginalize out, i.e., 
we want the energy \eqref{eq:e_sw} in the next window to not depend on $x_\im$ anymore.
Let $x_\ik$ be the states which share residuals with states in $x_\im$ (or have a prior), and $x_\iu$ those that are not directly connected to $x_\im$.
With that, we can rewrite $H$ and $b$ as
\begin{align}
\label{eq:hessian}
    H &= \begin{pmatrix}
    H_{\im\im} & H_{\im\ik} & 0 \\
    H_{\ik\im} & H_{\ik\ik}^\im+H_{\ik\ik}^{\bar{\im}} & H_{\ik\iu} \\
    0 & H_{\iu\ik} & H_{\iu\iu}
    \end{pmatrix}\,,
    \\
    b &= \begin{pmatrix}
    b_\im \\
    b_\ik^\im + b_\ik^{\bar{\im}} \\
    b_\iu
    \end{pmatrix}\,.
\end{align}
The Hessian block $H_{\ik\ik}$ is split into one part $H_{\ik\ik}^\im$ containing the Jacobians of residuals that depend on $\im$-variables,
and $H_{\ik\ik}^{\bar{\im}}$ that contains all others. The same holds for $b_\ik$.

\begin{figure*}
    \centering
    
    \input{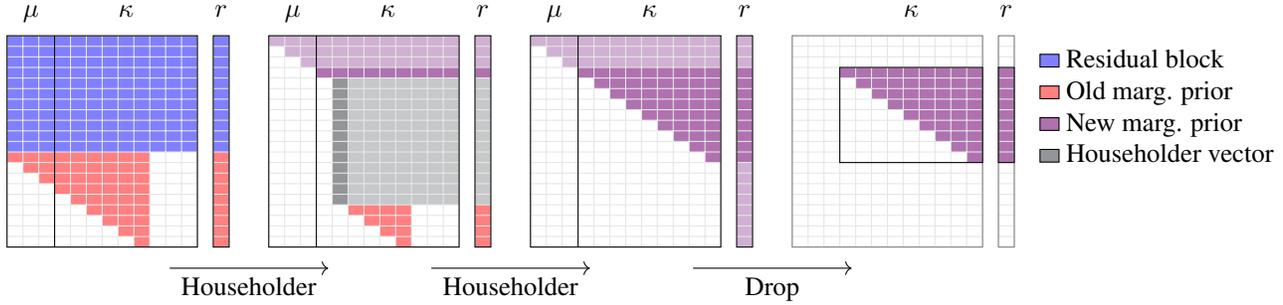}
    
    \caption{QR marginalization of frame variables. Initially, Jacobian and residual vector (after landmark marginalization) consist of residuals that depend on $\im$-variables and were active up to now and a marginalization prior (\emph{left}). By successively applying Householder transformations in-place, the matrix is transformed into an upper triangular matrix (\emph{flat}, semi-triangular in the rank-deficient case). Hereby, in each iteration, all elements except for the topmost element of the Householder vector vanish. To marginalize out the $\im$-variables, we drop the corresponding columns and the rows where these columns are non-zero, as well as zero rows. This results in a compact new marginalization prior (\emph{right}). Note that the old marginalization prior on the left depicts the case where always the oldest variables are marginalized. In practice, in order to have the $\mu$-columns on the left, variables may have to be reordered.}
    \label{fig:marginalization}
\end{figure*}

\subsubsection{Schur complement}
\label{sec:sc}

It can be shown by multiplication of the first line of \eqref{eq:hessian} with $H_{\ik\im}H_{\im\im}^{-1}$ that solving \eqref{eq:normal_eq} w.r.t.\ $x_\ik$ and $x_\iu$ is equivalent to solving the reduced system
\begin{equation}
\label{eq:reduced_system}
    \begin{pmatrix}
    \tilde{H} + H_{\ik\ik}^{\bar{\im}} & H_{\ik\iu} \\
    H_{\iu\ik} & H_{\iu\iu}
    \end{pmatrix}
    \begin{pmatrix}
    \Delta x_\ik \\
    \Delta x_\iu
    \end{pmatrix}
    =-
    \begin{pmatrix}
    \tilde{b} + b_{\ik}^{\bar{\im}} \\
    b_{\iu}
    \end{pmatrix}\,,
\end{equation}
with
\begin{align}
\label{eq:hmarg}
    \tilde{H} &= H_{\ik\ik}^\im - H_{\ik\im}H_{\im\im}^{-1}H_{\im\ik}\,,\\
\label{eq:bmarg}
    \tilde{b} &= b_{\ik}^\im - H_{\ik\im}H_{\im\im}^{-1}b_{\im}\,.
\end{align}
$\tilde{H}$ is called the \emph{Schur complement} of $H_{\im\im}$.
$\tilde{H}$ and $\tilde{b}$ only involve terms that depend on residuals containing $\im$-variables.
Thus, they do not change if new residuals depending on $\ik$- or $\iu$ variables are added to the energy.
In fact, even then, the solution of the reduced system will still be the same as if the full system including new residuals was solved.

$\tilde{H}$ and $\tilde{b}$ have been calculated using a linearization at the current state estimate $x$. In order to keep the system consistent, the linearization point $x_\ik^0$ of the $\ik$-variables may not be changed after computing $\tilde{H}$ and $\tilde{b}$ for a given $x^0$.
These so-called first-estimates Jacobians \cite{Guoquan09fej} prevent the destruction of nullspaces and are commonly used \cite{leutenegger2015keyframe, engel2017direct, usenko2019visual}.
Thus, to write the new marginalization energy as in~\eqref{eq:emarghess}, we use~\eqref{eq:shift} to obtain $H_\marg = \tilde{H}$ and $b_\marg = \tilde{b} - \tilde{H}(x_\ik-x_\ik^0)$.

In practice, $H_\marg$ and $b_\marg$ are computed once after the optimization of the old window has converged, using only those energy terms that depend on $\im$-variables (always including the old prior). The result constitutes the marginalizaiton prior used in the new window.
To preserve sparsity in the landmark-landmark Hessian block, we drop observations of landmarks that will stay active in frames which are about to be marginalized, before calculating $\tilde{H}$ and $\tilde{b}$. This means that landmarks are never part of the $\ik$-variables.

\section{Square root marginalization}
\label{sec:sqrtmarg}

Marginalization as presented in Sec.~\ref{sec:marginalization} is a very elegant way to keep the system size small while continuously adding new residual terms.
However, the implementation using the Schur complement has some drawbacks when it comes to numerical stability, e.g., the condition number of the Hessian being squared compared to the Jacobian.
In~\cite{demmel2021square} it was shown that a square root formulation for temporary marginalization of variables can be beneficial for speed, accuracy and numerical stability.
We now apply similar ideas to marginalization in sliding-window bundle adjustment.

\subsection{Landmark marginalization}

One option would be to permanently marginalize landmark and frame variables in one step.
However, we know from \cite{demmel2021square} that we can use nullspace projection to exploit the special sparsity structure of the Jacobian for landmark marginalization. On the other hand, nullspace projection is not optimal for permanent marginalization of frame variables (see below). 
We therefore choose a two-step procedure: first, we marginalize landmarks in $x_\im$ by projecting the Jacobians onto the nullspace $Q_2$ of the landmark Jacobian using QR decomposition as in~\cite{demmel2021square}.
In a second step we marginalize the frame variables in $x_\im$, taking the projected Jacobians $Q_2^\top J$ and residuals $Q_2^\top r$ as input.

\subsection{Frame variable marginalization}

\begin{figure}
    \centering
    \input{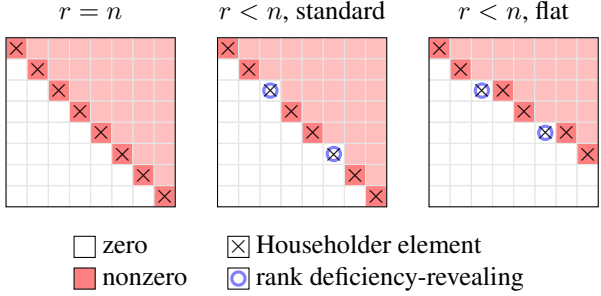}
    \caption{Our specialized flat QR decomposition. \emph{Left}: for a full-rank matrix, QR decomposition results in an upper triangular matrix $R$ with non-zero elements on the diagonal. \emph{Middle}: if the rank $r$ is smaller than the maximum possible rank $n$, standard Householder QR
    results in zero elements on the diagonal. \emph{Right}: with flat QR, when a zero diagonal element occurs, the Householder element for the next column will be in the same row.}
    \label{fig:rank_qr}
\end{figure}

In order to obtain a marginalization prior in square root form, we could proceed similar to~\cite{demmel2021square} and QR-decompose the Jacobian $J_\im$.
However, we want to keep the size of the marginalization prior as small as possible and we aim for a more general solution including the case where $J_\im$ (or $J$) does not have full rank.
To achieve this, we decompose the Jacobian $\begin{pmatrix}J_\im & J_\ik\end{pmatrix} = J = QR$ with a specialized QR algorithm that uses Householder reflections and is rank-revealing without any pivoting. Fig.~\ref{fig:marginalization} illustrates the whole procedure in the full-rank case.

Standard Householder QR~\cite[p. 248]{golub13} zeros the entries below the Householder element on the diagonal, column by column.
This may lead to a ``step'' of height larger than 1 in the resulting $R$ when that column is not linearly independent of the ones before and the Householder element remains zero (illustrated in Fig.~\ref{fig:rank_qr}, \emph{middle}).
In our specialized QR, if this occurs at matrix element $j_{ik}$, we proceed to the next column $k+1$ and instead of zeroing everything below element $j_{i+1,k+1}$ like in standard QR, we keep the row index at $i$ and zero everything below $j_{i,k+1}$ (see Fig.~\ref{fig:rank_qr}, \emph{right}).
Thus, we never get ``steps'' higher than one row.

The result is a valid QR decomposition with two appealing properties: when $J$ has full rank, it is equivalent to standard Householder QR. In addition, our QR algorithm is rank-revealing without any pivoting, the rank $r$ of $J$ being the number of non-zero rows in $R$.
The rank $r_\im$ of $J_\im$ is the number of non-zero rows in $Q^\top J_\im$, i.e., the first $n_\im$ columns in $Q^\top J$.
Our $R$ is in general more \emph{flat} than that obtained by other pivoting-free QR algorithms.

For marginalization, we take $R$ and drop the first $r_\im$ rows, the first $n_\im$ columns, and all zero rows at the bottom.
The remaining matrix, which we call $\tilde{R}$ in the following, is a square root of the marginalization Hessian $\tilde{H}$, as will be shown in the following subsections.
It has the same number $n_\ik$ of columns as $\tilde{H}$, while its number of rows may be smaller and is equal to $\rk(\tilde{H})$.
Thus, we also fulfil the need for small prior size, and can directly read off the rank of the prior, which is not as easy for SC-based marginalization.
The definition of the new marginalization energy as in \eqref{eq:emargsqrt} is now given by
    $J_\marg = \tilde{R}$ and $r_\marg = \tilde{r} - \tilde{R}(x-x^0)$,
where $\tilde{r}$ contains those elements of $Q^\top r$ whose according rows were not dropped in $R$ and $x$ and $x^0$ already have marginalized variables removed.

\subsection{Equivalence to SC-based marginalization}

We will only show the equivalence of our flat QR decomposition for permanent marginalization to that of Schur complement, since a proof for the equivalence of nullspace marginalization for temporary marginalization can be found in~\cite{demmel2021square}. We concentrate on the equivalence of $\tilde{R}, \tilde{r}$ in square root form and $\tilde{H}, \tilde{b}$ in Hessian form. The shifts on $\tilde{r}$ and $\tilde{b}$ in order to obtain $r_\marg$ and $b_\marg$ are then equivalent by \eqref{eq:shift}.

\subsubsection{Full-rank Jacobian}

Since $Q$ in $J=QR$ is an orthogonal matrix, $H=J^\top J=R^\top R$, and $R$ is a square root of the full Hessian.
We can write $R$ as a block matrix and define an orthogonally transformed residual vector $r'=Q^\top r$ such that $b=R^\top r'$:
\begin{align}
\label{eq:Jr}
    \begin{pmatrix}
    R_{1\im} & R_{1\ik} \\
    0 & R_{2\ik}
    \end{pmatrix}\,,\quad
    \begin{pmatrix}
    r'_1 \\
    r'_2
    \end{pmatrix}\,.
\end{align}
When $J$ is a full-rank matrix, $R_{1\im}$ and $R_{2\ik}$ are upper triangular matrices of sizes $n_\im\times n_\im$ and $(N-n_\im)\times n_\ik$ with a number $N$ of residuals (after landmark nullspace marginalization).
The upper triangular property implies that the bottom $N-n_\ik$ rows of $R_{2\ik}$ are all zero, so
\begin{equation}
\label{eq:r_marg}
    R_{2\ik} = \begin{pmatrix}
    \tilde{R} \\ 0
    \end{pmatrix}\,.
\end{equation}
Note, that a matrix $R$ with the same properties (but in general not equal) can also be obtained by LLT or LDLT decomposition of the Hessian sub-blocks related to $\{\im,\ik\}$, by setting $\tilde{R}=L^\top$ or $\tilde{R}=D^\frac{1}{2}L^\top$.
So we can even store the marginalization prior in square root form after computing the Hessian, without performing any QR decomposition.
We now show that marginalization with Schur complement is equivalent to keeping $R_{2\ik}^\top$ and $r'_2$, and dropping the remaining components of $R$ and $r'$.
Using \eqref{eq:Jr}, we obtain
\begin{align}
    \begin{split}
    H_{\im\im} &= R_{1\im}^\top R_{1\im}\,,\\
    H_{\im\ik} &= R_{1\im}^\top R_{1\ik}\,,\\
    b_\im &= R_{1\im}^\top r'_1\,,
    \end{split}
    \begin{split}
    H_{\ik\ik} &= R_{1\ik}^\top R_{1\ik} + R_{2\ik}^\top R_{2\ik}\,,\\
    H_{\ik\im} &= R_{1\ik}^\top R_{1\im}\,,\\
    b_\ik &= R_{1\ik}^\top r'_1 + R_{2\ik}^\top r'_2\,,
    \end{split}
\end{align}
where from now, $H_{\ik\ik}$ and $b_{\ik}$ denote $H_{\ik\ik}^\im$ and $b_{\ik}^\im$.
Using these expressions and the assumption that $R_{1\im}$ is invertible, \eqref{eq:hmarg} and \eqref{eq:bmarg} can be transformed as
\begin{align}
\begin{split}
    \tilde{H} &= R_{1\ik}^\top R_{1\ik} + R_{2\ik}^\top R_{2\ik}\\ &\quad- R_{1\ik}^\top R_{1\im}(R_{1\im}^\top R_{1\im})^{-1}R_{1\im}^\top R_{1\ik}\\
    &= R_{1\ik}^\top R_{1\ik} + R_{2\ik}^\top R_{2\ik} - R_{1\ik}^\top R_{1\ik}\\
    &= R_{2\ik}^\top R_{2\ik} = \tilde{R}^\top\tilde{R}\,,
\end{split}
\end{align}
and
\begin{align}
\begin{split}
    \tilde{b} &= R_{1\ik}^\top r'_1 + R_{2\ik}^\top r'_2 - R_{1\ik}^\top R_{1\im}(R_{1\im}^\top R_{1\im})^{-1} R_{1\im}^\top r'_1\\
    &= R_{1\ik}^\top r'_1 + R_{2\ik}^\top r'_2 - R_{1\ik}^\top r'_1\\
    &= R_{2\ik}^\top r'_2 = \tilde{R}^\top\tilde{r}\,.
\end{split}
\end{align}
The last equality in both equations is due to the fact that $R_{2\ik}$ is empty in the lower part, see~\eqref{eq:r_marg}.

\subsubsection{Rank-deficient Jacobian}
\label{sec:rank_deficient_jac}

Let us now assume that the Jacobian $J$ of our problem, and thus the Hessian $H$, are rank-deficient, i.e., their rank is smaller than the total number of parameters we optimize.
This can, for example, happen when the system is not properly constrained by an initial prior or when the prior becomes smaller than the numeric noise during operation.

We perform our rank-revealing, pivoting-free QR decomposition to compute an orthogonal $Q=\begin{pmatrix}Q_1&Q_2\end{pmatrix}$ with $Q_1\in\mathbb{R}^{N\times r_\im}$, $Q_2\in\mathbb{R}^{N\times (N-r_\im)}$ and $r_\im=\rk(J_\im)$ such that the matrix
\begin{equation}
    R = \begin{pmatrix}
    R_{1\im} & R_{1\ik} \\
    0 & R_{2\ik}
    \end{pmatrix}
    =
    \begin{pmatrix}
    Q_1^\top \\ Q_2^\top
    \end{pmatrix}
    \begin{pmatrix}
    J_\im & J_\ik
    \end{pmatrix}
\end{equation}
has a zero block on the lower left and $R_{1\im}\in\mathbb{R}^{r_\im\times n_\im}$.
Marginalization is done in the same way as for the full-rank case, i.e., we drop $R_{1\im}$ and $R_{1\ik}$, and only keep $R_{2\ik}$, minus all zero rows at the bottom, ending up with a matrix $\tilde{R}$.
The marginalization Hessian obtained this way, $\tilde{R}^\top \tilde{R}$, is the same as a pseudo-Schur complement, where in \eqref{eq:hmarg},~\eqref{eq:bmarg}, the inverse is replaced by a pseudo-inverse:
\begin{align}
    \label{eq:generalized_schur}
    \tilde{H} &= H_{\ik\ik} - H_{\ik\im}H_{\im\im}^+H_{\im\ik}\,,\\
    \label{eq:generalized_b}
    \tilde{b} &= b_{\ik\ik} - H_{\ik\im}H_{\im\im}^+b_{\im}\,.
\end{align}
To show this equivalence, we first rewrite
\begin{equation}
\begin{aligned}
    \tilde{R}^\top \tilde{R} &= R_{2\ik}^\top R_{2\ik} = J_\ik^\top Q_2Q_2^\top J_\ik \\
    &= J_\ik^\top J_\ik - J_\ik^\top Q_1Q_1^\top J_\ik \\
    &= H_{\ik\ik} - J_\ik^\top Q_1Q_1^\top J_\ik\,.
\end{aligned}
\end{equation}
On the other hand, for the pseudo-Schur complement, we can use the compact SVD decomposition $J_\im=U_1D_1V_1^\top$ (i.e., $D_1$ is of size $r_\im\times r_\im$, see Appendix) to obtain
\begin{equation}
\begin{aligned}
\label{eq:pseudomarg}
    &H_{\ik\ik} - H_{\ik\im}H_{\im\im}^+ H_{\im\ik}
      = H_{\ik\ik} - J_{\ik}^\top J_{\im} (J_\im^\top J_\im)^+ J_\im^\top J_{\ik} \\
    & = H_{\ik\ik} - J_{\ik}^\top (U_1D_1V_1^\top) (V_1D_1^{-2}V_1^\top) (V_1D_1U_1^\top) J_{\ik} \\
    & = H_{\ik\ik} - J_{\ik}^\top U_1U_1^\top J_{\ik}
      = H_{\ik\ik} - J_\ik^\top Q_1Q_1^\top J_\ik\,,
\end{aligned}
\end{equation}
and similarly for $\tilde{b}$. The last equality is due to the fact that $Q_1$ and $U_1$ span the same $r_\im$-dimensional subspace of $\mathbb{R}^N$, namely the space spanned by the columns of $J_\im$ (see Appendix).
Thus, using our flat QR decomposition, we obtain the same solution as by using pseudo-SC, but without the need for computing the pseudo-inverse.

\begin{figure*}[t]
\begin{center}
   \includegraphics[width=0.99\linewidth]{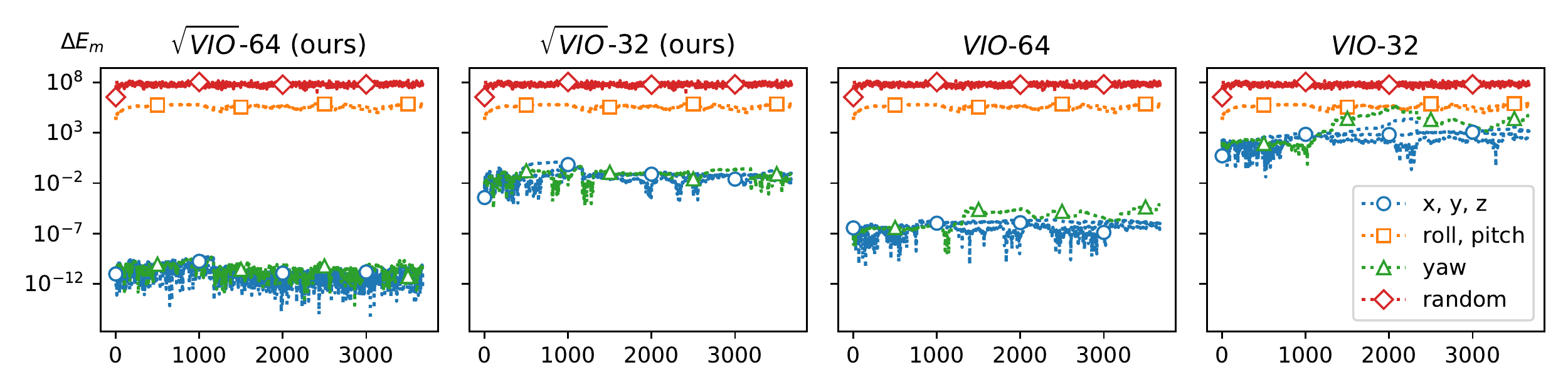}
\end{center}
\caption{In visual-inertial odometry we expect 4 degrees of gauge freedom.
To confirm this for the marginalized residuals, we investigate the marginalization prior cost change when perturbing the linearization point by a global translation in x, y, or z, by a global (linearized) rotation roll, pitch, or yaw, or by a random unit norm vector. 
While our square root marginalization leads to a consistent prior with expected nullspaces for both single and double precision, in the conventional squared form accumulating error leads to inconsistency. 
At around frame 2000 it appears that yaw is erroneously as observable as roll and pitch for \sqvio[32].}
\label{fig:ns}
\end{figure*}

\subsubsection{Interpreting the generalized Schur complement}
\label{sec:rank_deficiency}

We adopt the commonly used generalization of the Schur complement technique that simply replaces matrix inverses by pseudo-inverses in the case where the Jacobian is not of full rank~\cite{leutenegger2015keyframe, mazuran2016nonlinear}.
In the following, we argue why this is a good idea:
first, we define $\Delta x_{\text{tot}}$ as the solution we get by computing $-H^+b$,
and let $\Delta x_{\text{red}}$ be the solution that we obtain by using generalized $\tilde{H}$ and $\tilde{b}$ as in \eqref{eq:generalized_schur},~\eqref{eq:generalized_b} and solving the reduced system \eqref{eq:reduced_system} with pseudo-inverse, (possibly) followed by back substitution for $\Delta x_\im$.
Then we can use Lemma 2.3 from~\cite{liu2010generalized} to find that if
\begin{equation}
\label{eq:lemma23}
    \rk(J_\im) + \rk\begin{pmatrix}J_\ik & J_\iu\end{pmatrix} = \rk(J)\,,
\end{equation}
i.e., if the Jacobian sub-blocks corresponding to $\im$ and $\{\ik,\iu\}$ are not coupled by linearly dependent columns,
the two solutions are the same, see Appendix for a proof.

Of all possible $\Delta x$ that satisfy $H\Delta x=-b$, the solution $\Delta x_{\text{tot}}$ obtained via $H^+$ is the one with smallest norm $\Vert \Delta x\Vert$.
In practice, this is usually the preferred solution, as we don't want the system to drift in unobservable directions.

In case \eqref{eq:lemma23} does not hold, e.g.\ in the presence of absolute pose ambiguity, additional residuals may remove the coupled rank deficiency, and the pseudo-Schur complement approach yields the same solution for the reduced system as if a minimum-norm solution for the whole system including the new energy terms was calculated.

Note that while the property of a minimum-norm solution is very appealing, the computation of pseudo-inverses often uses SVD decomposition, which may be slow.
Faster matrix decomposition techniques that can deal with singular matrices, e.g.\ LU or LDLT, may output solutions $\Delta x$ with $\Vert\Delta x\Vert > \Vert\Delta x_{\text{tot}}\Vert$.
Our flat QR combines the advantages of both: it produces the same minimum-norm solution as generalized SC \emph{and} can be efficiently implemented.

\begin{table}
\footnotesize
\setlength{\tabcolsep}{0.3em}
\begin{center}
\begin{tabular}[t]{lcccc}%
\toprule
&\sqrtvio[64]&\sqrtvio[32]&\sqvio[64]&\sqvio[32]\\%
\midrule%
eurocMH01&\textbf{0.093}&\textbf{0.093}&\textbf{0.093}&\textit{0.991}\\%
eurocMH02&\textbf{0.048}&\textbf{0.048}&\textbf{0.048}&\textbf{0.048}\\%
eurocMH03&\textbf{0.051}&\textbf{0.051}&\textbf{0.051}&x\\%
eurocMH04&\textbf{0.109}&\textbf{0.109}&\textbf{0.109}&x\\%
eurocMH05&\textbf{0.137}&\textbf{0.137}&\textbf{0.137}&x\\%
eurocV101&\textbf{0.043}&\textbf{0.043}&\textbf{0.043}&\textbf{0.043}\\%
eurocV102&\textbf{0.048}&\textbf{0.048}&\textbf{0.048}&\textbf{0.048}\\%
eurocV103&\textbf{0.058}&\textbf{0.058}&\textbf{0.058}&x\\%
eurocV201&\textbf{0.037}&\textbf{0.037}&\textbf{0.037}&\textbf{0.037}\\%
eurocV202&\textbf{0.053}&\textbf{0.053}&\textbf{0.053}&x\\%
tumvi{-}corr1&\textbf{0.300}&\textbf{0.300}&\textbf{0.300}&x\\%
tumvi{-}corr2&\textbf{0.426}&\textbf{0.426}&\textbf{0.426}&x\\%
tumvi{-}mag1&\textbf{1.456}&\textit{1.457}&\textbf{1.456}&x\\%
tumvi{-}mag2&\textbf{0.908}&\textit{0.920}&\textbf{0.908}&x\\%
tumvi{-}room1&\textbf{0.102}&\textbf{0.102}&\textbf{0.102}&\textit{0.104}\\%
tumvi{-}room2&\textbf{0.071}&\textbf{0.071}&\textbf{0.071}&x\\%
tumvi{-}slides1&\textbf{0.310}&\textbf{0.310}&\textbf{0.310}&x\\%
tumvi{-}slides2&\textbf{0.759}&\textbf{0.759}&\textbf{0.759}&x\\%
\bottomrule
\end{tabular}
\end{center}
\caption{
Absolute trajectory error (in meters) for VIO shows that in contrast to the baseline, the proposed approach also works with floating point precision providing essentially the same accuracy.
}
\label{tab:vio_ate}
\end{table}

\section{Evaluation}

We base our implementation on the open-source odometry \emph{Basalt},
which is a highly efficient state-of-the-art marginalizing sliding-window odometry with KLT feature tracking as frontend and SC-based optimization and marginalization in the backend ~\cite{usenko2019visual}.
In a minor adaptation for slightly increased performance, we marginalize a feature whose track has been lost right away, and not only together with its host frame.
The NS-projection for eliminating landmarks is inspired by \cite{demmel2021square}, but since during optimization we solve a small system with only up to 7 keyframes and a few hundred landmarks, we make several adjustments that improve runtime: 
instead of conjugate gradient to solve the reduced camera system (RCS), we explicitly form normal equations and solve with LDLT, and we skip Jacobian scaling and landmark damping. 
Another alternative for solving the RCS that avoids normal equations would be QR decomposition, which we found to have higher runtime.
We implement everything in one codebase, ensuring maximum comparability: the only difference between compared variants is the choice of prior storage (squared vs square root), and the algorithm for optimization and marginalization.

\begin{table}
\footnotesize
\setlength{\tabcolsep}{0.3em}
\begin{center}
\begin{tabular}[t]{lcccc}%
\toprule
&\sqrtvo[64]&\sqrtvo[32]&\sqvo[64]&\sqvo[32]\\%
\midrule%
kitti00&\textbf{3.92}&\textbf{3.92}&\textbf{3.92}&x\\%
kitti02&\textbf{9.72}&\textbf{9.72}&\textbf{9.72}&x\\%
kitti03&\textbf{1.34}&\textbf{1.34}&\textbf{1.34}&\textbf{1.34}\\%
kitti04&\textbf{1.22}&\textbf{1.22}&\textbf{1.22}&\textbf{1.22}\\%
kitti05&\textbf{2.75}&\textbf{2.75}&\textbf{2.75}&x\\%
kitti06&\textbf{2.61}&\textbf{2.61}&\textbf{2.61}&\textbf{2.61}\\%
kitti07&\textit{1.52}&1.53&\textit{1.52}&\textbf{1.44}\\%
kitti08&\textbf{3.85}&\textbf{3.85}&\textbf{3.85}&x\\%
kitti09&\textbf{4.13}&\textbf{4.13}&\textbf{4.13}&x\\%
kitti10&\textbf{1.11}&\textbf{1.11}&\textbf{1.11}&\textit{26.12}\\%
\bottomrule
\end{tabular}
\end{center}
\caption{
Absolute trajectory error (in meters) for VO shows the same tendency as VIO in Tab.~\ref{tab:vio_ate}.
}
\label{tab:vo_ate}
\end{table}

All variants make good use of multi-threading and use a state-of-the-art dense linear algebra library \cite{eigenweb}.
Floating point precision is indicated by suffixes, such as \sqrtvio[64] or \sqvo[32].
The experiments are run on an Ubuntu 18.04 desktop with 64GB RAM and an Intel Xeon W-2133 with 12 virtual cores at 3.60GHz.
VIO is evaluated on the EuRoC MAV dataset \cite{burri2016euroc} and a subset of TUMVI \cite{schubert2018tumvi}.
For the KITTI odometry benchmark (training set) \cite{geiger2012kitti} we evaluate VO, since it does not have synchronized IMU data (\emph{kitti01} is excluded since the optical flow of \cite{usenko2019visual} fails).
Additional results can be found in the Appendix.

\subsection{Accuracy and runtime}

We evaluate the accuracy of pose estimation with the Absolute Trajectory Error (ATE), the translational RMSE of camera positions after $\SE(3)$ alignment to the ground truth~ \cite{sturm2012tumrgbd}. Tab.~\ref{tab:vio_ate} and Tab.~\ref{tab:vo_ate} show the ATE for VIO and VO, respectively, both for the square root and the squared implementation.
It can be seen, that in double precision both variants result in the same accuracy, as does \sqrtvio[32], while \sqvio[32] often fails numerically or results in very high ATE (same for VO). Qualitatively, the SC-based single-precision estimators perform fine initially, but usually diverge soon (see Fig.~\ref{fig:teaser} \emph{top}).

The biggest portion of total runtime is spent on optimization. There, using NS-projection is 22\% faster in single precision compared to double, while for SC the speedup is only 8\%.
This is because for NS-projection we can do dense linear algebra operations on larger matrices, 
while an efficient SC implementation needs to exploit sparsity and operate explicitly on small matrix blocks. The larger matrix operations benefit more from SIMD instructions of modern CPUs.
In total runtime, the proposed \sqrtvio[32] is $36\%$ faster than the baseline \sqvio[64] (see Tab.~\ref{tab:vio_runtime}).

\begin{table}
\footnotesize
\setlength{\tabcolsep}{0.6em}
\begin{center}
\begin{tabular}[t]{lcccc}%
\toprule
&\sqrtvio[64]&\sqrtvio[32]&\sqvio[64]&\sqvio[32]\\%
\midrule%
eurocMH01&\textit{23.4}%
~/~%
2.5&\textbf{18.6}%
~/~%
2.3&35.9%
~/~%
\textit{1.8}&33.4%
~/~%
\textbf{1.7}\\%
eurocMH02&\textit{20.0}%
~/~%
2.1&\textbf{15.6}%
~/~%
1.9&31.7%
~/~%
\textit{1.5}&29.0%
~/~%
\textbf{1.4}\\%
eurocMH03&\textit{17.6}%
~/~%
1.8&\textbf{13.9}%
~/~%
\textit{1.6}&26.3%
~/~%
\textbf{1.3}&x\\%
eurocMH04&\textit{13.1}%
~/~%
1.3&\textbf{10.3}%
~/~%
\textit{1.2}&19.5%
~/~%
\textbf{0.9}&x\\%
eurocMH05&\textit{15.0}%
~/~%
1.5&\textbf{11.6}%
~/~%
\textit{1.3}&22.6%
~/~%
\textbf{1.1}&x\\%
eurocV101&\textit{15.0}%
~/~%
2.2&\textbf{12.0}%
~/~%
\textit{2.0}&23.6%
~/~%
\textbf{1.5}&22.4%
~/~%
\textbf{1.5}\\%
eurocV102&\textit{8.3}%
~/~%
1.0&\textbf{6.8}%
~/~%
\textit{0.9}&11.5%
~/~%
\textbf{0.7}&10.6%
~/~%
\textbf{0.7}\\%
eurocV103&\textit{8.3}%
~/~%
1.0&\textbf{6.7}%
~/~%
\textit{0.9}&11.1%
~/~%
\textbf{0.7}&x\\%
eurocV201&\textit{12.1}%
~/~%
\textit{1.4}&\textbf{9.5}%
~/~%
\textit{1.4}&20.8%
~/~%
\textbf{1.0}&19.2%
~/~%
\textbf{1.0}\\%
eurocV202&\textit{11.4}%
~/~%
1.3&\textbf{9.3}%
~/~%
\textit{1.2}&15.5%
~/~%
\textbf{0.9}&x\\%
tumvi{-}corr1&\textit{24.4}%
~/~%
3.2&\textbf{18.7}%
~/~%
\textit{2.6}&36.7%
~/~%
\textbf{2.2}&x\\%
tumvi{-}corr2&\textit{29.4}%
~/~%
3.8&\textbf{22.0}%
~/~%
\textit{3.1}&42.2%
~/~%
\textbf{2.6}&x\\%
tumvi{-}mag1&\textit{78.1}%
~/~%
10.5&\textbf{57.4}%
~/~%
\textit{8.4}&112.5%
~/~%
\textbf{7.0}&x\\%
tumvi{-}mag2&\textit{59.6}%
~/~%
7.7&\textbf{42.2}%
~/~%
\textit{6.3}&88.2%
~/~%
\textbf{5.1}&x\\%
tumvi{-}room1&\textit{13.2}%
~/~%
1.7&\textbf{10.0}%
~/~%
\textit{1.4}&21.6%
~/~%
\textbf{1.3}&19.6%
~/~%
\textbf{1.3}\\%
tumvi{-}room2&\textit{12.2}%
~/~%
1.8&\textbf{9.4}%
~/~%
\textit{1.5}&20.2%
~/~%
\textbf{1.3}&x\\%
tumvi{-}slides1&\textit{28.6}%
~/~%
3.6&\textbf{20.9}%
~/~%
\textit{3.0}&44.1%
~/~%
\textbf{2.5}&x\\%
tumvi{-}slides2&\textit{24.8}%
~/~%
3.1&\textbf{18.5}%
~/~%
\textit{2.5}&38.8%
~/~%
\textbf{2.1}&x\\%
\bottomrule
\end{tabular}
\end{center}
\caption{Total runtime in seconds spent on ``optimization /  marginalization'' in VIO. 
\emph{Optimization:} with NS-projection for landmarks (\sqrtvio[32]) is almost twice as fast as the baseline using SC (\sqvio[64]). 
\emph{Marginalization:} Conventional SC is a bit faster, but this step only takes a small fraction of the overall runtime. 
}
\label{tab:vio_runtime}
\end{table}

In an ablation study, we combine the square root prior form (\ref{eq:emargsqrt}) with different optimization (\emph{NS+LDLT} and \emph{SC+LDLT}) and marginalization variants (\emph{NS+QR} and \emph{SC+SC}). 
Here, SC-marginalization is always immediately followed by factorizing the prior into square root form with LDLT decomposition. We can see that this factorization alone is not enough to prevent the severe degradation of accuracy in single precision.
Only the combination of all proposed improvements leads to the best accuracy and runtime (see Tab.~\ref{tab:ablation}).

\begin{table}
\footnotesize
\setlength{\tabcolsep}{0.3em}
\begin{center}
\begin{tabular}{lccccccccc}%
\toprule
  & \multicolumn{2}{c}{proposed} && \multicolumn{6}{c}{ablation study} \\%
  \cmidrule{2-3} \cmidrule{5-10}
  opt. & \multicolumn{2}{c}{NS+LDLT} && \multicolumn{2}{c}{SC+LDLT} & \multicolumn{2}{c}{NS+LDLT} & \multicolumn{2}{c}{SC+LDLT}\\%
  marg. & \multicolumn{2}{c}{NS+QR} && \multicolumn{2}{c}{NS+QR} & \multicolumn{2}{c}{SC+SC} & \multicolumn{2}{c}{SC+SC} \\%
  precision & 64 & 32 && 64 & 32 & 64 & 32 & 64 & 32\\%
  \midrule%
ATE {[}m{]}&\textbf{0.068}&\textbf{0.068}&&\textbf{0.068}&\textbf{0.068}&\textbf{0.068}&0.232&\textbf{0.068}&\textit{0.211}\\%
real{-}time&6.9x&\textbf{8.2x}&&5.0x&5.6x&7.1x&\textit{7.9x}&5.2x&5.5x\\%
t total {[}s{]}&17.9&\textbf{14.9}&&24.4&21.8&17.4&\textit{15.5}&23.7&22.2\\%
t opt {[}s{]}&14.4&\textbf{11.4}&&22.2&20.3&14.4&\textit{11.5}&22.1&20.4\\%
t marg {[}s{]}&1.6&1.5&&1.6&\textit{1.3}&1.4&1.4&\textit{1.3}&\textbf{1.2}\\%
\bottomrule
\end{tabular}

\end{center}
\caption{Different combinations of optimization and marginalization techniques, and floating-point precision for \sqrtvio{} on EuRoC.
All variants store the marginalization prior in square root form (\ref{eq:emargsqrt}).
The shown metrics (ATE, runtime: total / optimization / marginalization) are averages over all sequences, 
and the real-time factor indicates how much faster the processing is compared to sequence duration.
The proposed square root marginalization \emph{NS+QR} is deciding for good accuracy in single precision, 
while the square root optimization \emph{NS+LDLT} leads to best runtime.
}
\label{tab:ablation}
\end{table}

\subsection{Numerical properties of marginalization prior}
\label{sec:numerical_properties}

Analytically, $H_\marg$ is positive semi-definite and has a nullspace equivalent to the non-linear system (which is one of the reasons we use first-estimates Jacobians, see Sec.~\ref{sec:sc}). 
This nullspace contains (at least) the dimensions corresponding to the global gauge freedom of the system: global translation and yaw for VIO, and additionally roll and pitch for VO. 
Note, that while for optimzation we add an absolute pose prior to fix the gauge, here we consider $H_\marg$ without such additional prior. 
We analyse $H_\marg$ numerically by looking at its smallest eigenvalue $\sigma_{\min}$ and at the change in prior cost $\Delta E_\marg$ for a state perturbation $\epsilon$ around the linearization point in the direction of the expected gauge freedom, with $||\epsilon|| = 1$:
\begin{equation}
    \Delta E_\marg = E_\marg(x^0 + \epsilon) - E_\marg(x^0) = \tfrac{1}{2}\epsilon^\top H_\marg \epsilon + \epsilon^\top b_\marg\,.
\end{equation}
Note, that eigenvalues and cost change are always computed after converting the prior to double, and for the square root estimator we compute $H_\marg = J_\marg^\top J_\marg$.
Fig.~\ref{fig:teaser} (\emph{bottom}) shows that with a squared formulation in single precision we get either negative eigenvalues with large magnitude (indefinite prior) or a large positive minimum eigenvalue (vanishing nullspace). 
Similarly, we observe in Fig.~\ref{fig:ns} that the gauge freedom appears to vanish.
The proposed \sqrtvio[32] and \sqrtvo[32] suffer from neither of these problems and are numerically stable and thus retain full accuracy.

\section{Conclusion}

We introduced a square root sliding-window bundle adjustment approach that is well suited for real-time visual and visual-inertial odometry applications.  The method combines elimination of landmark variables using nullspace projection with a matrix square root of the Hessian for storing the maginalization prior which is in turn directly updated using a specialized QR decomposition. We proved that the specialized QR decomposition is (analytically) equivalent to Schur complement. Yet, experimental evaluation on a range of real-world datasets reveals that the proposed approach is $23\%$ faster than the baseline.  Moreover, in contrast to the baseline approach, the proposed method remains numerically stable when run in single floating point precision, leading to a combined speedup of $36\%$ while preserving the same accuracy and robustness.

{\small
\bibliographystyle{ieee_fullname}
\bibliography{references}
}

\includecomment{commentappendix}

\begin{commentappendix}

\clearpage

\appendix

\section{Proofs and mathematical properties}

\subsection{Pseudo-Schur complement and SVD}

First, we show the properties we use in Section~\ref{sec:rank_deficient_jac} when proving that our proposed specialized QR decomposition is equivalent to using pseudo-Schur complement.
This includes in particular the definition of the compact SVD for a rank-deficient matrix together with a definition of the Moore-Penrose inverse, as well as a result on subspaces of $\mathbb{R}^n$ spanned by matrix columns.

\begin{defn}
Let $J\in\mathbb{R}^{n\times k}$ and $\rk(J)=r\leq k$.
The \emph{compact} singular value decomposition (SVD) of $J$ is of the form
\begin{equation}
    J = U_1 D_1 V_1^\top\,,
\end{equation}
where $U_1\in \mathbb{R}^{n\times r}$, $D_1\in \mathbb{R}^{r\times r}$, and $V_1\in \mathbb{R}^{k\times r}$.
$D_1$ is an invertible diagonal matrix with positive entries, $U_1^\top U_1=V_1^\top V_1=I_r$.
\end{defn}

Thus, by definition of the compact SVD, the columns of $U_1$ span the column space of $J$.
For the compact SVD of $J^\top J$, we get $V_1D_1^2V_1^\top$.

\begin{defn}
The \emph{Moore-Penrose inverse} (also \emph{pseudo-inverse}) of a matrix with compact SVD $U_1D_1V_1^\top$ is defined as
\begin{equation}
    (U_1D_1V_1^\top)^+ = V_1 D_1^{-1}U_1^\top\,.
\end{equation}
\end{defn}
Thus, the pseudo-inverse of $(J^\top J)$ is given by
\begin{equation}
    (J^\top J)^+ = V_1 D_1^{-2}V_1^\top\,.
\end{equation}

\begin{lemma}
Let $Q,U\in\mathbb{R}^{n\times r}$, and let the columns of $Q$ and $U$ span the same $r$-dimensional subspace of $\mathbb{R}^n$.
Further, let both $Q$ and $U$ have mutually orthogonal columns of norm 1, i.e., $Q^\top Q=U^\top U=I_r$.
Then, the following holds:
\begin{equation}
    QQ^\top = UU^\top\,.
\end{equation}
\end{lemma}
\begin{proof}
Since the columns of $Q$ and $U$ span the same space, each column of $Q$ can be written as a linear combination of the columns of $U$ and vice versa.
Thus, there is a matrix $M$ such that $Q^\top = MU^\top$ and $U^\top = M^{-1}Q^\top$.
As $Q^\top Q=U^\top U=I_r$, $M=Q^\top U$ and $M^{-1}=U^\top Q=M^\top$.
Thus, $M$ is orthogonal, yielding
\begin{equation}
    QQ^\top = UM^\top M U^\top = UU^\top\,.
\end{equation}
\end{proof}

\begin{figure*}[t]
\begin{center}
   \includegraphics[width=0.99\linewidth]{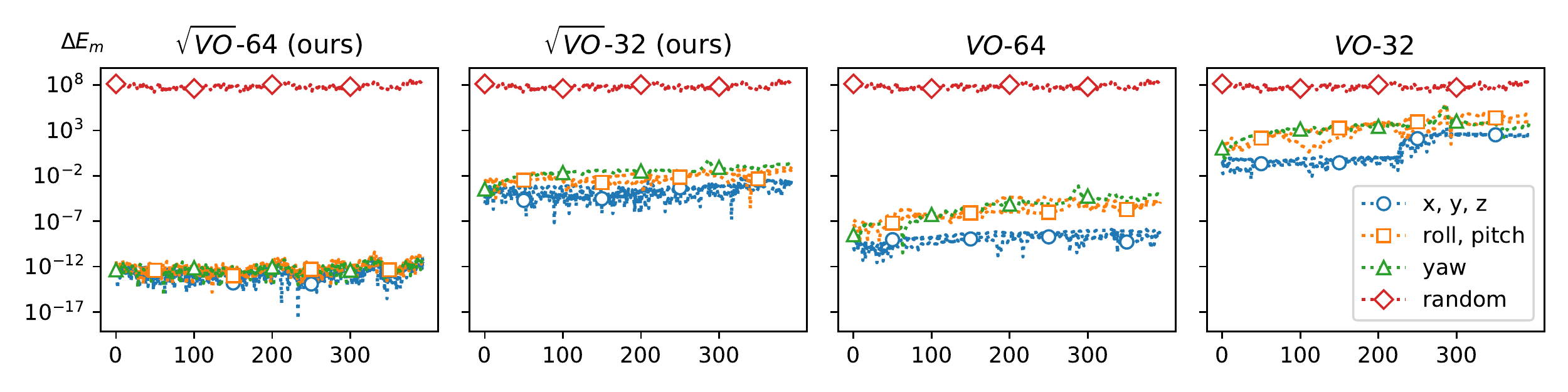}
\end{center}
\caption{
Where for visual-inertial odometry we expect 4 degrees of gauge freedom, for stereo visual odometry roll and pitch are not observable and we expect 6 degrees of gauge freedom.
The plots show the marginalization prior cost change $\Delta E_\marg$ for VO on \emph{kitti10} when perturbing the linearization point.
For that, we consider perturbations by a global translation (in x, y, or z), by a global (linearized) rotation (roll, pitch, or yaw), or by a random unit norm vector.
While our square root marginalization leads to a consistent prior with expected nullspaces for both single and double precision, in the conventional squared form accumulating error leads to inconsistency.
Similar to the VIO case (compare Fig.~\ref{fig:ns}), here for \sqvo[32] the prior over time erroneously appears to make the global pose observable, indicated by large cost change by perturbations in gauge direction.
In particular, after around 200 keyframes there is a noticeable increase, which also coincides with worsened pose estimation (see Fig.~\ref{fig:vo_trajectory}).
}
\label{fig:vo_ns}
\end{figure*}

\subsection{Equivalence of pseudo-inverse and pseudo-Schur complement}

In Sec.~\ref{sec:rank_deficiency}, we claim that under certain conditions, solving the full system using Moore-Penrose inverse is equivalent to using the generalized Schur complement followed by solving the reduced system with Moore-Penrose inverse. Moreover, a potential backsubstitution for the $\im$-variables can also be achieved using a Moore-Penrose inverse instead of an inverse: 
\begin{equation}
\label{eq:x_mu_red}
\Delta x_{\im,\text{red}} = H_{\im\im}^+(b_\im - H_{\im\ik}\Delta x_{\ik,\text{red}})\,.
\end{equation}
In the following, we will formalize and prove this statement.

\begin{thm}
Let \eqref{eq:lemma23} hold, and let $\Delta x_{\text{tot}}$ and $\Delta x_{\text{red}}$ be defined as in Sec.~\ref{sec:rank_deficiency}. Then,
\begin{equation}
    \Delta x_{\text{tot}} = \Delta x_\text{red}\,.
\end{equation}
\end{thm}

\begin{proof}
We start by noting that
\begin{align}
    \rk(J_\im) &= \rk(H_{\im\im})\,,\\
    \rk\begin{pmatrix}J_\ik & J_\iu\end{pmatrix} &=
    \rk\begin{pmatrix}H_{\ik\ik} & H_{\ik\iu} \\ H_{\iu\ik} & H_{\iu\iu} \end{pmatrix} =: r_{\ik\iu}\,, \\
    \rk(J) &= \rk(H)\,.
\end{align}
Thus, we can rewrite \eqref{eq:lemma23} as
\begin{equation}
    \rk(H) = \rk(H_{\im\im}) + r_{\ik\iu}\,,
\end{equation}
and apply Lemma 2.3 from~\cite{liu2010generalized} with $A_{11} = H_{\im\im}$.
This Lemma gives us a block-matrix expression for the pseudo-inverse $H^+$ of $H$:
\begin{align}
    H^+ &= \begin{pmatrix}
    A & -B^\top \\ -B & S^+
    \end{pmatrix}\,, \\
    \label{eq:fullSchur}
    S &= \begin{pmatrix}\tilde{H}+H_{\ik\ik}^{\bar{\im}} & H_{\ik\iu} \\
    H_{\iu\ik} & H_{\iu\iu}\end{pmatrix}\,, \\
    B &= S^+\begin{pmatrix}H_{\ik\im} \\ 0\end{pmatrix}
    H_{\im\im}^+\,, \\
    A &= H_{\im\im}^+ + H_{\im\im}^+
    \begin{pmatrix}H_{\im\ik} & 0\end{pmatrix}
    B\,.
\end{align}
If we now compute $-H^+b$ and look at the $\ik$- and $\iu$-components, we get
\begin{equation}
\label{eq:proof_ku}
\begin{aligned}
    &\begin{pmatrix}
    \Delta x_{\ik,\text{tot}} \\
    \Delta x_{\iu,\text{tot}}
    \end{pmatrix} =
    Bb_\im - S^+ \begin{pmatrix}b_\ik \\ b_\iu\end{pmatrix} \\
    &= S^+\begin{pmatrix}H_{\ik\im}H_{\im\im}^+b_\im - b_\ik \\ -b_\iu
    \end{pmatrix}
    = -S^+\begin{pmatrix}\tilde{b} + b_\ik^{\bar{\im}} \\ b_\iu
    \end{pmatrix}\,,
\end{aligned}
\end{equation}
which is exactly the solution of \eqref{eq:reduced_system}, i.e.,
\begin{equation}
\label{eq:proof_ku2}
    \begin{pmatrix}
    \Delta x_{\ik,\text{tot}} \\
    \Delta x_{\iu,\text{tot}}
    \end{pmatrix} = 
    \begin{pmatrix}
    \Delta x_{\ik,\text{red}} \\
    \Delta x_{\iu,\text{red}}
    \end{pmatrix}
\end{equation}
Similarly, from
\begin{equation}
    \Delta x_{\im,\text{tot}} = -Ab_\im + B^\top\begin{pmatrix}b_\ik \\ b_\iu \end{pmatrix}\,,
\end{equation}
after some steps, one obtains the back substitution formula \eqref{eq:x_mu_red}
\begin{equation}
\label{eq:proof_m}
    \Delta x_{\im,\text{tot}} = H_{\im\im}^+(b_\im - H_{\im\ik}\Delta x_{\ik,\text{red}}) = \Delta x_{\im,\text{red}}\,.
\end{equation}
\eqref{eq:proof_ku2} and \eqref{eq:proof_m} together conclude the proof.
\end{proof}

\paragraph{Note on square root of the $\ik\iu$-system}

While we have shown that $\tilde{H}=\tilde{R}^\top\tilde{R}$ and $\tilde{b}=\tilde{R}^\top\tilde{r}$, to complete the square root formulation, a square root of the system including $u$-variables as in \eqref{eq:reduced_system} and \eqref{eq:fullSchur} is given by 
\begin{align}
    \tilde{R}_{\ik\iu} = \begin{pmatrix}
        J_{\ik}^{\bar{\im}} & J_\iu^{\bar{\im}} \\
        \tilde{R} & 0
    \end{pmatrix}\,,\qquad
    \tilde{r}_{\ik\iu} = \begin{pmatrix}
        r^{\bar{\im}} \\
        \tilde{r}
    \end{pmatrix}\,.
\end{align}

\begin{table}
\footnotesize
\setlength{\tabcolsep}{0.6em}
\begin{center}
\begin{tabular}[t]{lcccc}%
\toprule
&\sqrtvo[64]&\sqrtvo[32]&\sqvo[64]&\sqvo[32]\\%
\midrule%
kitti00&\textit{29.5}%
~/~%
2.7&\textbf{23.6}%
~/~%
\textbf{2.2}&50.2%
~/~%
\textit{2.3}&x\\%
kitti02&\textit{32.0}%
~/~%
3.0&\textbf{25.0}%
~/~%
\textbf{2.3}&53.2%
~/~%
\textit{2.4}&x\\%
kitti03&\textit{5.2}%
~/~%
\textit{0.6}&\textbf{4.3}%
~/~%
\textbf{0.5}&9.4%
~/~%
\textbf{0.5}&9.0%
~/~%
\textbf{0.5}\\%
kitti04&\textit{1.5}%
~/~%
\textit{0.2}&\textbf{1.2}%
~/~%
\textbf{0.1}&2.6%
~/~%
\textbf{0.1}&2.5%
~/~%
\textbf{0.1}\\%
kitti05&\textit{18.0}%
~/~%
1.7&\textbf{15.0}%
~/~%
\textbf{1.4}&31.1%
~/~%
\textit{1.5}&x\\%
kitti06&\textit{5.8}%
~/~%
\textit{0.6}&\textbf{4.8}%
~/~%
\textbf{0.5}&9.8%
~/~%
\textit{0.6}&9.3%
~/~%
\textit{0.6}\\%
kitti07&\textit{6.3}%
~/~%
\textit{0.7}&\textbf{5.3}%
~/~%
\textbf{0.6}&11.2%
~/~%
\textbf{0.6}&10.7%
~/~%
\textbf{0.6}\\%
kitti08&\textit{26.3}%
~/~%
2.5&\textbf{21.2}%
~/~%
\textbf{2.0}&44.2%
~/~%
\textit{2.1}&x\\%
kitti09&\textit{10.1}%
~/~%
\textit{1.0}&\textbf{8.0}%
~/~%
\textbf{0.8}&16.7%
~/~%
\textbf{0.8}&x\\%
kitti10&\textit{6.9}%
~/~%
\textit{0.7}&\textbf{5.5}%
~/~%
\textbf{0.6}&11.6%
~/~%
\textbf{0.6}&9.7%
~/~%
\textbf{0.6}\\%
\bottomrule
\end{tabular}

\end{center}
\caption{
Total runtime in seconds spent on ``optimization /  marginalization'' in VO. 
\emph{Optimization:} NS-projection for landmarks (\sqrtvo[32]) is almost twice as fast as the baseline using SC (\sqvo[64]). 
\emph{Marginalization:} conventional SC may be slightly faster, but this step only takes a small fraction of the overall runtime. 
}
\label{tab:vo_runtime}
\end{table}

\section{Additional analysis of VO results}

\begin{table}
\footnotesize
\setlength{\tabcolsep}{0.3em}
\begin{center}
\begin{tabular}{lccccccccc}%
  \toprule
    & \multicolumn{2}{c}{proposed} && \multicolumn{6}{c}{ablation study} \\%
  \cmidrule{2-3} \cmidrule{5-10}
  opt. & \multicolumn{2}{c}{NS+LDLT} && \multicolumn{2}{c}{SC+LDLT} & \multicolumn{2}{c}{NS+LDLT} & \multicolumn{2}{c}{SC+LDLT}\\%
  marg. & \multicolumn{2}{c}{NS+QR} && \multicolumn{2}{c}{NS+QR} & \multicolumn{2}{c}{SC+SC} & \multicolumn{2}{c}{SC+SC} \\%
  precision & 64 & 32 && 64 & 32 & 64 & 32 & 64 & 32\\%
  \midrule  
ATE {[}m{]}&\textbf{3.216}&\textbf{3.216}&&\textbf{3.216}&\textbf{3.216}&\textit{3.217}&3.293&\textbf{3.216}&3.479\\%
real{-}time&9.4x&\textbf{9.8x}&&8.0x&8.6x&9.4x&\textit{9.6x}&8.2x&8.6x\\%
t total {[}s{]}&24.3&\textit{23.3}&&28.7&26.5&24.1&\textbf{23.2}&28.2&26.5\\%
t opt {[}s{]}&14.2&\textit{11.3}&&24.2&22.6&14.1&\textbf{11.0}&24.2&22.2\\%
t marg {[}s{]}&1.4&\textit{1.1}&&1.2&\textbf{1.0}&1.4&1.3&1.2&1.2\\%
\bottomrule
\end{tabular}

\end{center}
\caption{
Different combinations of optimization and marginalization techniques, and floating-point precision for \sqrtvo{} on KITTI. 
All variants store the marginalization prior in square root form (\ref{eq:emargsqrt}).
The shown metrics (ATE, runtime: total / optimization / marginalization) are averages over all sequences, 
and the real-time factor indicates how much faster the processing is compared to sequence duration.
The proposed square root marginalization \emph{NS+QR} is deciding for good accuracy in single precision, 
while the square root optimization \emph{NS+LDLT} leads to best runtime.
}
\label{tab:vo_ablation}
\end{table}

\begin{figure}[t]
\begin{center}
   \includegraphics[width=0.99\linewidth]{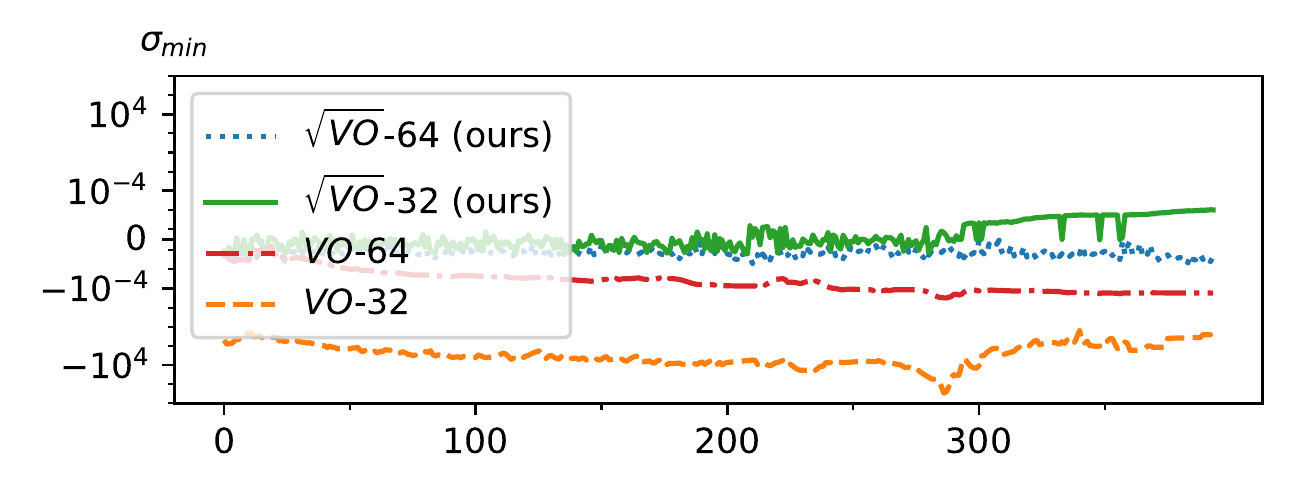}
\end{center}
\caption{
Smallest eigenvalue $\sigma_{\min}$ of the marginalization prior Hessian $H_\marg$ evolving over time for VO on \emph{kitti10} (linear y-axis for  $|\sigma_{\min}| < 10^{-8}$, logarithmic elsewhere).
We expect values close to zero (positive semi-definite Hessian with gauge freedom). While the conventional (squared) formulation in single precision leads to negative eigenvalues with large magnitude (exceeding $10^{8}$), accumulating error, and (ultimately) numerical failure, the proposed square root approach has $\sigma_{\min}$ of bounded magnitude (less than $10^{-4}$) and remains stable.
}
\label{fig:vo_ev}
\end{figure}

In this section we include additional results supporting the claims of the main paper. 
While these are for the same datasets, we expand upon some of the analysis that was omitted due to limited space.
Specifically, we show runtimes, the ablation study, as well as the nullspace and eigenvalue analysis also for VO on the KITTI dataset.
Qualitatively, these are similar to the VIO results from the main paper and thus we draw the same conclusions.

Tab.~\ref{tab:vo_runtime} shows runtimes for optimization and marginalization for VO (compare VIO results in Tab.~\ref{tab:vio_runtime}). It can be seen that optimization takes a much bigger portion of total runtime than marginalization, that for the proposed single-precision solver \sqrtvo[32] it is around twice as fast as the competing baseline \sqvo[64], and that the square root formulation benefits more in terms of runtime from switching from double to single precision.

Tab.~\ref{tab:vo_ablation} shows the same ablation study as Tab.~\ref{tab:ablation}, but for VO instead of VIO.
Note that for KITTI, the twofold improvement in optimization runtime is not fully reflected in an improvement of total runtime.
The reason is that here the optical flow, which is computed in a single parallel thread, becomes the bottleneck.
However, the improved optimization runtime still means the required compute power is reduced.
Overall, also for VO we conclude that only the combination of all proposed improvements leads to best accuracy and runtime.

The analysis of numerical properties of the marginalization prior Hessian of VO on \emph{kitti10} reveals similar behaviour to VIO (see Sec.~\ref{sec:numerical_properties}).
For the squared formulation in single precision the marginalization prior becomes numerically indefinite (Fig.~\ref{fig:vo_ev}, compare VIO results in Fig.~\ref{fig:teaser} \emph{bottom}) and gauge freedom vanishes (Fig.~\ref{fig:vo_ns}, compare VIO results in Fig.~\ref{fig:ns}).
While initially the pose estimation works fine, at some point the accumulating error leads to bad state estimates and ultimately numerical failure (Fig.~\ref{fig:vo_trajectory}, compare Fig.~\ref{fig:teaser} \emph{top}).
In contrast, the proposed \sqrtvo\ has the same accuracy in both single and double precision, at a significantly reduced computational cost.

\begin{figure}[t]
\begin{center}
   \includegraphics[width=0.99\linewidth]{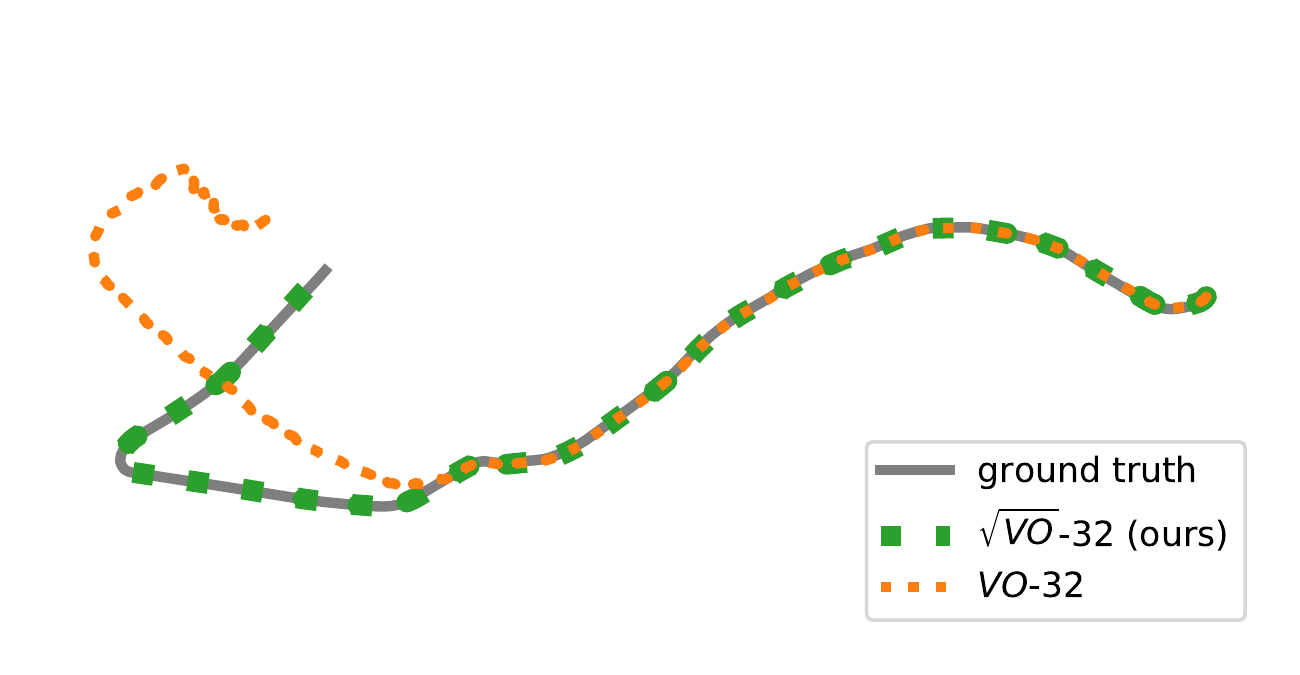}
\end{center}
\caption{
Estimated visual odometry trajectories on the \emph{kitti10} sequence. 
The conventional baseline \sqvo[64] works well with double-precision floats, but fails in single precision (\sqvo[32]). 
In contrast, the proposed square root estimator \sqrtvo[32] even in single precision retains the same accuracy.
}
\label{fig:vo_trajectory}
\end{figure}

\section{Notes on memory overhead}
The main memory requirement of our optimization and marginalization 
comes from the dense landmark blocks, where we perform
QR on the Jacobians in-place to marginalize landmarks. 
\cite{demmel2021square} reports around twice the memory use compared to SC for sparse BA problems 
and mentions memory to be the limiting factor for large dense problems. 
However, for us the number of keyframes and number of observations per landmark are bounded in the sliding window and thus memory use is not a major concern.

For example, for VIO on Euroc MH01 we have at most 4033 observations across all landmarks, and at most 7 keyframes (3 with IMU, state size 15, and 4 pose-only, state size 6), so the Jacobians have in total 8066 rows and 73 columns (3+1 extra for landmark+residual), giving an approximate upper bound of $2.4\textrm{MB}$ with 32bit floats. Measuring the actual difference in peak memory between the single and double precision variants reveals $1.3\textrm{MB}$ for the square root solver and $0.9\textrm{MB}$ for the SC solver, while the vast majority of process peak memory at around $300\textrm{MB}$ is spent in other parts of the (not memory-optimized) system (e.g. cached image queue, logging, etc...).

A memory-conscious implementation could in fact reduce the required landmark-block memory by doing a one-pass over landmarks that linearizes, marginalizes and accumulates the RCS Hessian using scratch memory. Only 3 rows per landmark for back-substitution would need to be stored.

\end{commentappendix}

\end{document}